\documentclass[11pt,letterpaper]{amsart}
\usepackage[utf8]{inputenc}
\usepackage{graphicx} 
\usepackage{amsmath,amssymb,amsthm, mathtools, tikz-cd, verbatim}
\usepackage{relsize}
\usepackage{biblatex}
\usepackage{geometry}
\usepackage{listings}
\usepackage{xcolor}
\usepackage{subcaption}

\definecolor{codegreen}{rgb}{0,0.6,0}
\definecolor{codegray}{rgb}{0.5,0.5,0.5}
\definecolor{codepurple}{rgb}{0.58,0,0.82}
\definecolor{backcolour}{rgb}{0.95,0.95,0.92}

\lstdefinestyle{mystyle}{
  backgroundcolor=\color{backcolour}, commentstyle=\color{codegreen},
  keywordstyle=\color{magenta},
  numberstyle=\tiny\color{codegray},
  stringstyle=\color{codepurple},
  basicstyle=\ttfamily\footnotesize,
  breakatwhitespace=false,         
  breaklines=true,                 
  captionpos=b,                    
  keepspaces=true,                 
  numbers=left,                    
  numbersep=5pt,                  
  showspaces=false,                
  showstringspaces=false,
  showtabs=false,                  
  tabsize=2
}

\title{A Complete Symmetry Classification of Shallow ReLU Networks}
\author{Pranavkrishnan Ramakrishnan}
\date{April 2026}

\begin{document}

\maketitle

\newcommand{\Null}{\textup{Null }}
\newcommand{\Leftnull}{\textup{Leftnull }}
\newcommand{\Col}{\textup{Col }}
\newcommand{\Row}{\textup{Row }}
\newcommand{\rank}{\textup{rank }}
\newcommand{\im}{\textup{im }}
\newcommand{\Span}{\textup{Span }}

\newcommand{\GL}{\textup{GL}}
\newcommand{\Cl}{\textup{Cl}}
\newcommand{\Stab}{\textup{Stab}}
\newcommand{\Fix}{\textup{Fix}}
\newcommand{\id}{\textup{id}}
\newcommand{\Orb}{\textup{Orb}}
\newcommand{\Inn}{\textup{Inn}}
\newcommand{\Aut}{\textup{Aut}}
\newcommand{\lcm}{\textup{lcm}}
\newcommand{\Syl}{\textup{Syl}}
\newcommand{\Hom}{\textup{Hom}}
\newcommand{\Sym}{\textup{Sym}}
\newcommand{\comb}[2]{\begin{pmatrix} #1 \\ #2 \end{pmatrix}}
\newcommand{\vtr}[2]{(\begin{smallmatrix} #1 \\ #2 \end{smallmatrix})}
\newcommand{\Vtr}[2]{\begin{bmatrix} #1 \\ #2 \end{bmatrix}}
\newcommand{\tor}{\textup{tor}}
\newcommand{\block}[1]{
  \underbrace{0 \; \cdots\; 0}_{#1}
}
\newcommand{\Gal}{\textup{Gal}}
\newcommand{\paramspace}[2]{\Omega_{(#1, \ldots, #2)}}
\newcommand{\diaggroup}[2]{\mathcal{H}_{(#1, \ldots, #2)}}
\newcommand{\realizationmap}[2]{\rho_{(#1, \ldots, #2)}}

\newtheorem{theorem}{Theorem}
\newtheorem{lemma}{Lemma}
\newtheorem{corollary}[lemma]{Corollary}
\newtheorem{proposition}[lemma]{Proposition}
\newtheorem{defin}{Definition}
\newtheorem{example}{Example}
\newtheorem{definlem}[defin]{Definition-Lemma}
\newtheorem{remark}{Remark}
\begin{abstract}
    Parameter space is not function space for neural network architectures. This fact, investigated as early as the 1990s under terms such as ``reverse engineering," or ``parameter identifiability", has led to the natural question of parameter space symmetries\textemdash the study of distinct parameters in neural architectures which realize the same function. Indeed, the quotient space obtained by identifying parameters giving rise to the same function, called the \textit{neuromanifold}, has been shown in some cases to have rich geometric properties, impacting optimization dynamics. Thus far, techniques towards complete classifications have required the analyticity of the activation function, notably excising the important case of ReLU. Here, in contrast, we exploit the non-differentiability of the ReLU activation to provide a complete classification of the symmetries in the shallow case.
\end{abstract}
\section{Introduction}
\par Let $H_n\leq \GL_n(\mathbb{R})$ denote the subgroup generated by permutations and diagonal matrices with positive coefficients. This subgroup constitutes the well understood symmetries on the parameter space of a shallow feed forward ReLU neural network, cf. \cite{zhao2023symmetriesflatminimaconserved, godfrey2023symmetriesdeeplearningmodels}. That is to say, given parameter space 
$$
\Omega_{(m,n,k)} = M_{k\times n}(\mathbb{R})\times M_{n\times m}(\mathbb{R})\times \mathbb{R}^{n+k}
$$
of architecture $(m,n,k)$, we may define the \textit{realization map} as follows\footnote{The notation used for the parameters in $\Omega_{(m,n,1)}$ is, of course, non-standard. A more typical notation in the literature would be taking a parameter $\theta$ as $(A_2, b_2, A_1, b_1)$, where $\rho_{(m,n,k)}(\theta)=A_2\sigma(A_1x+b_1)+b_2$. Though this notation is expedient in view of deeper neural networks, since we are considering just the shallow case we opt for the non-standard notation which better visually aligns with realized function.}
\begin{center}
    \begin{tikzcd}
\rho_{(m,n,k)}:\Omega_{(m,n,k)} \arrow[r]                         & {C^0(\mathbb{R}^{m}, \mathbb{R}^{k})}                     \\
{(M, A, b, c)} \arrow[r, maps to] & M\sigma (Ax+b)+c
\end{tikzcd}
\end{center}
where $\sigma$ denotes the ReLU activation function, which is defined as 
$$
\sigma\left(\begin{bmatrix} x_1 \\ \vdots \\ x_{n}\end{bmatrix}\right) = \begin{bmatrix} \max\{0, x_1\} \\ \vdots \\ \max\{0, x_{n}\}\end{bmatrix}.
$$
We see that $H_n$ acts on $\Omega_{(m,n,k)}$ by
$$
h\cdot(M, A, b, c) = (Mh^{-1}, hA, hb, c)
$$
such that $\rho_{(m, n,k)}$ is a $H_n$-invariant map. Here by symmetry, we mean not only a group action which preserves the value of a given loss function as in \cite[Def. 3.1]{zhao2023symmetriesflatminimaconserved}, but any equivalence of parameters which is invariant under the realization map; this is notably a stronger condition than the one just mentioned, and to distinguish the terms functional loss symmetry \cite[Def. 2.6]{zhao2025symmetryneuralnetworkparameter} and functional neural symmetry \cite[Def. 2.4]{zhao2025symmetryneuralnetworkparameter} may be used.\footnote{With the caveat that symmetry to us does not mean that the equivalence derives from a group action. In the language of \cite{godfrey2023symmetriesdeeplearningmodels}, such symmetries which derive from a group action as called symmetry via intertwiner group.} In particular, this means that for any $\theta\in \Omega_{(m,n,k)}$, the fibre $\rho_{(m,n,k)}^{-1}(\rho_{(m,n,k)}(\theta))$, which by abuse of notation we shall now denote $\rho^{-1}_{(m,n,k)}(\theta)$, admits a map from $H_n$ by $h\mapsto h\cdot \theta$. We may naturally ask if $\rho_{(m,n,k)}^{-1}(\theta)$ is ``larger" than $H_n$. In fact, as shown in Lemmas \ref{lem:weakcounterexample} and \ref{lem:strongcounterexample}, the fibre $\rho_{(m,n,k)}^{-1}(\theta)$ is indeed not merely $H_n\cdot \theta$ for certain choices of $\theta$, but contains additional parameters not in the orbit. Called ``hidden symmetries" in \cite[Def. F.4]{pmlr-v202-grigsby23a}, the study of such additional symmetries has been a topic of interest in the mathematics of machine learning. Even as early as 1988, \cite{Chua1988CanonicalPR} investigates the idea of "canonical representations" of piecewise linear functions, and in the 1993, the neural network symmetry question was investigated in \cite{NIPS1993_e49b8b40_Fefferman_Market} and \cite{ALBERTINI1993975}, though with sigmoidal activations in mind. Likewise, a number of works in the past years study the case of polynomial activations, as in \cite{Finkel_2025, kileel2019expressivepowerdeeppolynomial}. As for the case of ReLU activations, there has been progress in restricting the picture to a neighbourhood to a given parameter \cite{Elisenda_Grigsby_2025, Stock_2022, bonapellissier2022localidentifiabilitydeeprelu} or in considering certain types of symmetries \cite{PetzkaHenning2Layer, zhao2023symmetriesflatminimaconserved}. There has also been particular progress in the case of deep ReLU neural networks, as in \cite{pmlr-v202-grigsby23a, phuong2020functional, bonapellissier2026geometryinducedregularizationdeeprelu}. Perhaps the most general statements thus far are in \cite{pmlr-v119-rolnick20a, zhang2026completeidentificationdeeprelu}, which provides general statements for parameters of deep neural networks, though there is also the very recent work investigating shallow cases in \cite{grillo2026relunetworksadmitidentifiable, gegenfurtner2026symmetriesthreelayerrelunetworks}.
\par The condition of high depth or polynomial activation belies a certain reliance of probabilistic or geometric techniques in the methodology, even epistemology, of the work, with a notable exception of \cite{zhang2026completeidentificationdeeprelu}, which uses techniques from Łukasiewicz logic. Here, however, we take an ostensibly algebraic approach to the case of shallow ReLU neural networks, which proves itself fruitful in the following main result:
\begin{theorem}\label{thm:maintheorem}
    The set of parameters $\theta\in \Omega_{(m,n,k)}$ where the fibre $ \rho_{(m,n,k)}^{-1}(\theta)$ is diffeomorphic to the Lie group $H_{n}$ is dense. 
\end{theorem}
\par Our method is to look at the case when $k=1$ and build up to the general case in a natural way. The proof in the $k=1$ requires the definition of a certain \textit{minimal form} (Definition \ref{def:minimalform}) associated to certain equivalence classes of our parameters, through which we are given sufficient data about the symmetries of our parameters to determine the behaviour of their fibres, and which is also computable via a program as given in Section \ref{section:minimalformcode} (it also bears mention that the related \textit{reduced form} bears some similarity to the reduced ReLU representation as defined in \cite{fu2025toricgeometryreluneural}.) Indeed, this formation of a minimal form gives us not only information on a dense subset of $\Omega_{(m,n,1)}$ (namely the one prescribed in Theorem \ref{thm:maintheorem},) but also gives us its full symmetry classification, by which we also get the full symmetry classification of $\Omega_{(m,n,k)}$ \textemdash this is the content of Theorems \ref{thm:symmetryclassificationk=1} and \ref{thm:symmetryclassification}. The heart of this classification is Proposition \ref{thm:nontrivialzeroparam}, which aims to provide conditions for when a certain type of parameter, minimal parameters with no 0-factors (Definition \ref{def:0factor}), realize a smooth function, that is to say an affine linear function, with an eye towards the 0 function. The proof of this proposition, which relies on checking everywhere differentiability on the domain, speaks to the utility of considering ReLU, which contains a point of non-differentiability at 0. We get as a result that beyond any of the immediate symmetries listed in Definition \ref{def:minimalform}, there remains only one class of symmetries in the case of shallow ReLU networks, which derives from the relationship that $\sigma(x)-\sigma(-x)=x$ (note that this is essentially the only nontrivial relation with the ReLU activation function which produces a smooth function.) In fact, the aforementioned immediate symmetries are not unique to the ReLU case, but are one example of symmetries prevalent for a broad class of neural networks. There is likely much to be said about this and the applications of these methods to a more general neural network, but for now we shall resign ourselves to Remark \ref{remark:generalizing}. 
\par We also prove a somewhat tangential result in classifying the possible stabilizers of a parameter under the $H_n$ action, which we call $\Stab_{H_n}\theta$. Specifically, we get the following result:
\begin{theorem}\label{thm:mainstabtheorem}
    For any given parameter $\theta\in\Omega_{(m,n,k)}$
    $$
    \Stab_{H_n} \theta \simeq S \times H_{n'}
    $$
    where $n'\leq n$ and $S\subseteq S_{n-n'}$ is a subgroup generated by at most $\comb{n-n'}{2}$ two cycles.
\end{theorem}
The details of this result and its exact form is given in Proposition \ref{thm:stabtheoremmnk}. It is natural to ask if the structure of the stabilizer of $\theta$ determines whether its fibre is diffeomorphic to $H_n$. To this extent, the utility of this theorem within the context of this work is showing that $\Stab_{H_n}\theta\neq 1$ is a sufficient yet not necessary condition for $\rho_{(m,n,k)}^{-1}(\theta)\not\simeq H_n$. As a result, we are given the following hierarchy of parameters: 
$$
\{\theta : \Stab_{H_n}\theta\neq 1\} \stackrel{\textup{Definition}}{\subset} \{\theta = (M, A, b,c) : \Stab_{H_n}(A, b)\neq 1\} \stackrel{\textup{Corollary }\ref{cor:StabilizerAbAndFibre}}{\subset} \{\theta: \rho_{(m,n,k)}^{-1}(\theta)\not\simeq H_n\}.
$$
That being said, however, Corollary \ref{cor:StabilizerAbAndFibre} itself becomes useful by allowing our method of reduction in Proposition \ref{prop:genericityreductionprop}, which leads to Theorem \ref{thm:maintheorem}. On account of this and their otherwise tangential nature, the stabilizer results shall form the contents of Section \ref{section:stabilizers}. 
\par We see that these classifications and results are given firmly in the language of (linear) algebra. It is natural to ask how these conditions might be interpreted geometrically. While one might with relative ease make geometric sense of the equivalences given in Definition \ref{def:minimalform}, a complete geometric interpretation of the last condition of $\equiv$ in Theorem \ref{thm:symmetryclassificationk=1} is not by any means immediate, and it is possible that this one algebraic equivalence encapsulates many different types of geometric equivalences. Alas, any substantive discussion of this are beyond the scope of the work. 
\par But to whet our appetites, let us consider the following seemingly innocuous parameter in $\Omega_{(2,3,1)}$:
$$
\theta_1=\left(\begin{bmatrix}
    1\\1\\1
\end{bmatrix}, \begin{bmatrix}
    1 & 1\\
    -1 & 0 \\
    0 & -1
\end{bmatrix}, \begin{bmatrix}
    -2 \\ -1 \\ -1
\end{bmatrix},-6\right)
$$
(Here, as indicated in the Notation section, we shall denote the matrix $\mathbb{R}^n\rightarrow \mathbb{R}$ as a vector which acts by its transpose on $\mathbb{R}^n$.) Indeed, this parameter satisfies the conditions of \cite[Lem. D15]{pmlr-v202-grigsby23a}: it is generic in that all $k$-fold intersections are affine linear subspaces of dimension $n-k$\cite[Sec. B]{pmlr-v202-grigsby23a}, and it is supertransversal \cite[Def. 11]{masden2022algorithmicdeterminationcombinatorialstructure} in that the bent hyperplane arrangement on $\mathbb{R}^m$ is transverse on cells\cite[Def. 4]{masden2022algorithmicdeterminationcombinatorialstructure} to the complex of the second layer map, in addition to satisfying both the LRA and TPIC conditions (See Figure \ref{fig:benthyperplanearrangement1}.)
\begin{figure}
\centering
\begin{subfigure}{.4\textwidth}
  \centering
  \includegraphics[width=.5\linewidth]{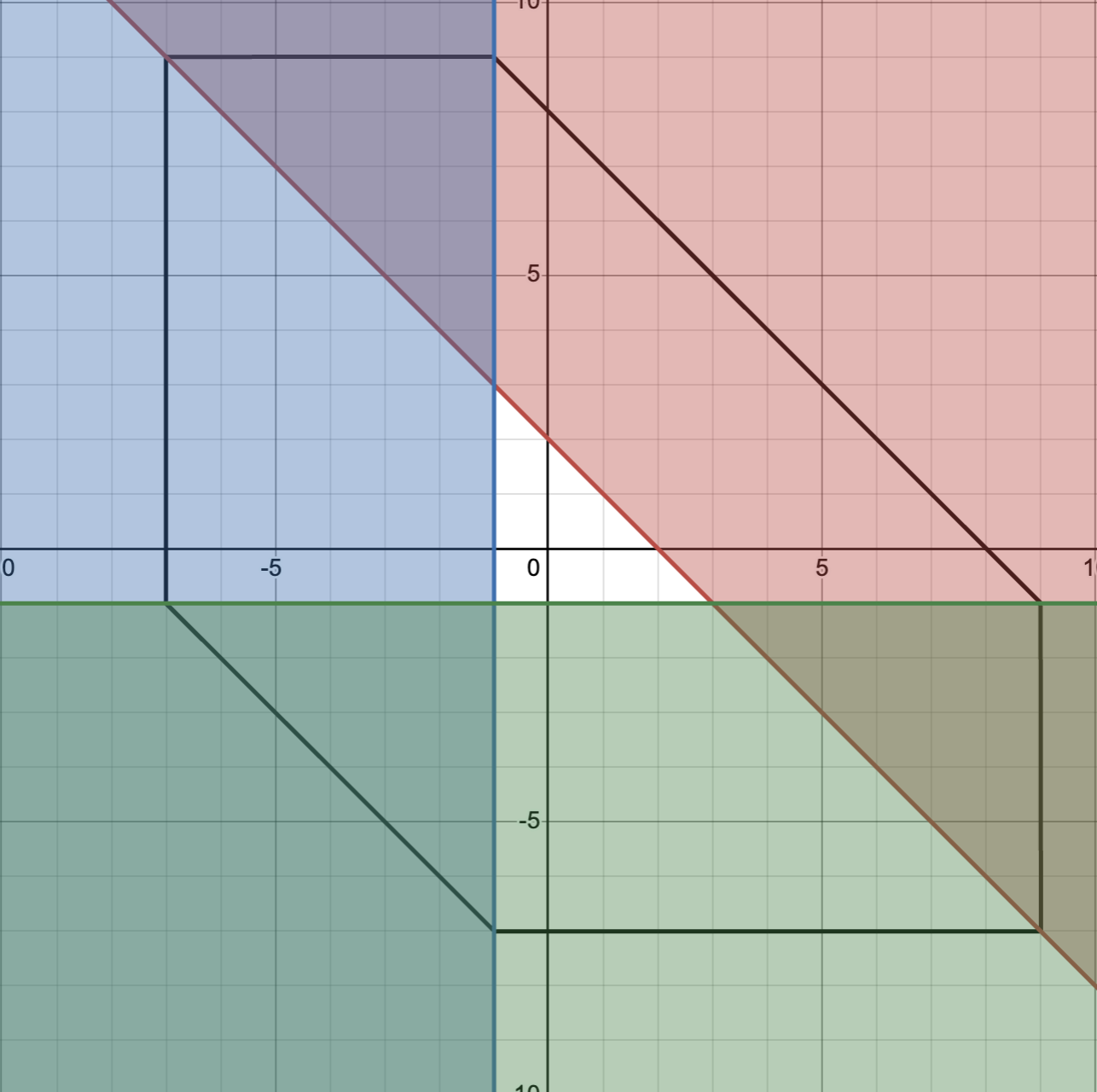}
  \caption{Bent hyperplane arrangement of $\theta_1$}
  \label{fig:benthyperplanearrangement1}
\end{subfigure}
\begin{subfigure}{.4\textwidth}
  \centering
  \includegraphics[width=.5\linewidth]{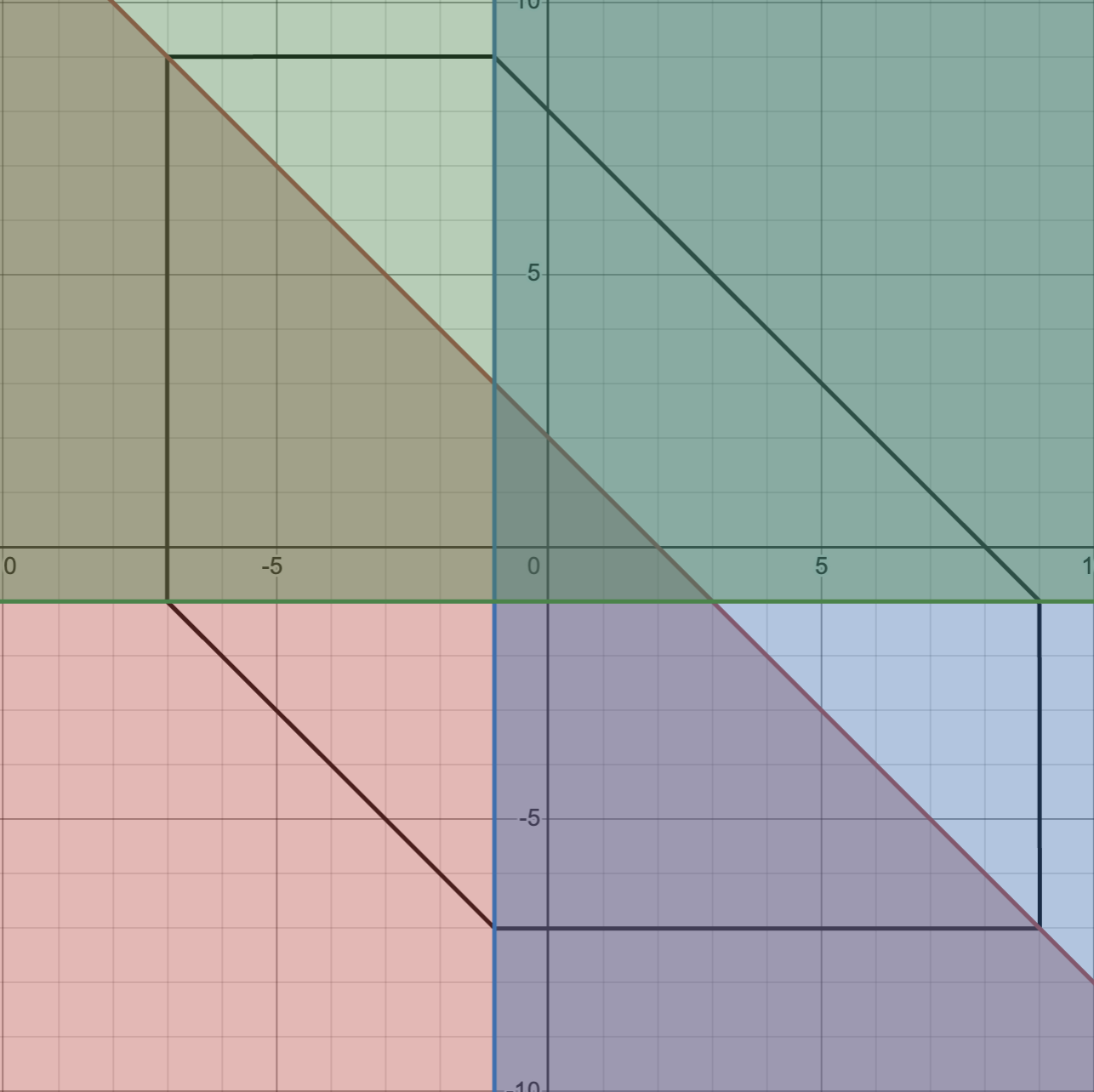}
  \caption{Bent hyperplane arrangement of $\theta_2$}
  \label{fig:benthyperplanearrangement2}
\end{subfigure}
\caption{The bent hyperplane arrangements of $\theta_1$ and $\theta_2$. The section in red represents the domain where the first hidden neuron is activated, blue the second, and green the third. Finally the set $\rho_{(2,3,1)}(\theta_i)(x,y)=0$ is denoted in black.}
\label{fig:benthyperplanearrangement}
\end{figure}
\par Yet, in contrast to the statement of \cite[Lem. D15]{pmlr-v202-grigsby23a}, the fibre $\rho_{(2,3,1)}^{-1}(\theta_1)$ is not determined by positive scalings and permutations of $\theta_1$ as  $\rho_{(2,3,1)}(\theta_1)=\rho_{(2,3,1)}(\theta_2)$ where
$$
\theta_2=\left(\begin{bmatrix}
    1\\1\\1
\end{bmatrix}, \begin{bmatrix}
    -1 & -1\\
    1 & 0 \\
    0 & 1
\end{bmatrix}, \begin{bmatrix}
    2 \\ 1 \\ 1
\end{bmatrix},-10\right).
$$
This equality derives from the identity $\sigma(x)-\sigma(-x)=x$ by the following
\begin{align*}
    \rho_{(2,3,1)}(\theta_1)-\rho_{(2,3,1)}(\theta_2) &= \sigma(x_1+x_2-2)+\sigma(-x_1-1) + \sigma(-x_2-1) - 6 \\
    &- (\sigma(-x_1-x_2+2)+\sigma(x_1+1) + \sigma(x_2+1) - 10)\\
    &= x_1+x_2 - 2 -x_1 - 1 -x_2 - 1-6+10\\
    &= 0
\end{align*}
(symmetries from the identity $\sigma(x)-\sigma(-x)=x$ is further discussed in Section \ref{section:mn1case}.) This shows that for a given parameter $\theta\in \Omega_{(m,n,k)}$, it satisfying the conditions of the aforementioned Lemma is not strong enough to show that $\rho_{(m,n,k)}^{-1}(\theta)\simeq H_n$. And so at best such conditions at best show \textit{local identifiability} \cite[Def. 6]{bonapellissier2022localidentifiabilitydeeprelu}. Whether it is sufficient to show local identifiability shall remain as conjecture for this work. 
\vskip 5pt
{\flushleft \bf Acknowledgements:} The core of this works stems from discussions with and attending the class of Kathryn Lindsey during the spring and fall of 2025, which extended into a research project with Eli Grigsby starting in earnest in spring of 2025. For this, I would be remiss not to express here my immense gratitude. I would also like to express my thanks in particular for the discussions I was able to have Joe Boninger, Qile Chen, Yaoying Fu, Spencer Leslie, and Ming Hong Tee, from which this paper has immensely benefitted. This work was partially supported by the Institute for Foundations of Machine Learning (IFML).
\section{Notation and Preliminaries}
\subsection{Notation}
\par To recall the notation used in the introduction, $H_n\leq \GL_n(\mathbb{R})$ denotes the subgroup generated by permutations, which we shall denote $S_n$ and diagonal matrices with positive coefficients, which we shall denote $D_n^+$. Likewise, $\Omega_{(m,n,k)} = M_{k\times n}(\mathbb{R})\times M_{n\times m}(\mathbb{R})\times \mathbb{R}^{n+k}$ denotes parameter space of architecture $(m,n,k)$, and the ReLU activation function $\sigma:\mathbb{R}^l\rightarrow \mathbb{R}^l$ is the function given by
$$
\sigma\left(\begin{bmatrix} x_1 \\ \vdots \\ x_{l}\end{bmatrix}\right) = \begin{bmatrix} \max\{0, x_1\} \\ \vdots \\ \max\{0, x_{l}\}\end{bmatrix},
$$
where the domain, and thereby codomain, is clear from context. A parameter in $\Omega_{(m,n,k)}$ is denoted by $\theta=(M, A, b, c)$, where $M\in M_{k\times n}(\mathbb{R}), A\in  M_{n\times m}(\mathbb{R}), b\in \mathbb{R}^{n}, c\in \mathbb{R}^k$, and admits an $H_n$-action denoted by $\cdot$ and given by
$$
h\cdot(M, A, b, c) = (Mh^{-1}, hA, hb, c).
$$
When $k=1$, in place of $M$ we shall write $v$ and in place of $c$ we shall write the capital $C$. Moreover, in the case where a parameter has no bias, that is to say is of the form $(M, A, 0,0)$, we may forego stating the bias and just write $(M, A)$ (likewise $(v, A)$ in place of $(v, A, 0,0)$.) The realization map $\rho_{(m,n,k)}:\Omega_{(m,n,k)}\rightarrow C^0(\mathbb{R}^m, \mathbb{R}^k)$ is the $H_n$-invariant map given by 
$$
\rho_{(m,n,k)}(M, A, b, c) = M\sigma(Ax+b)+c,
$$
or if $k=1$,
$$
\rho_{(m,n,1)}(v, A, b, C) = v^T\sigma(Ax+b)+C = \langle v,\sigma(Ax+b)\rangle+C.
$$
The fibre of $\theta$, denoted by $\rho^{-1}_{(m,n,k)}(\theta)$, is  the set $\rho^{-1}_{(m,n,k)}(\rho_{(m,n,k)}(\theta))$. 
\par In addition, we also make use of a projection map $\pi_i:\Omega_{(m,n,k)}\rightarrow \Omega_{(m,n,1)}$ throughout the work: for
$$
    (M, A, b, c) = \left(\begin{bmatrix}
        m_1 \\ \vdots\\ m_k
    \end{bmatrix}, A, b, \begin{bmatrix}
        c_1 \\ \vdots\\ c_k
    \end{bmatrix}\right)\in \Omega_{(m,n,k)},
$$
we define $\pi_i(M, A, b, c) = (m_i, A, b, c_i)$. 
\par Throughout Section \ref{section:stabilizers}, for some $G$-set $X, \Stab_G x$ denotes the stabilizer of $x\in X$ under the $G$-action (We usually take $X=\Omega_{(m,n,k)}$ and $G= H_n, S_n, D_n^+$.) Moreover, when used to denote an element of $\textup{GL}_n(\mathbb{R})$, we let $(i,j)$ denote the permutation of the $i$-th and $j$-th coordinate, and let
$$
\textup{diag}(\lambda_1, \ldots, \lambda_n) = \begin{bmatrix}
    \lambda_1 & & \\
    & \ddots & \\
    & & \lambda_n
\end{bmatrix}
$$
and denote $\textup{diag}(\lambda_{i_1}, \ldots, \lambda_{i_m})$ the matrix with $\lambda_{i_j}$ on the $i_j$-th entry and 1 otherwise. 
\par Finally, in Section 4, in accordance with notation from Commutative Algebra and Algebraic Geometry, we shall denote
$$
V(f) = \{x: f(x)=0\}
$$
\subsection{Preliminaries from the theory of Lie groups} Here we shall recall some relevant results from the theory of Lie groups that shall be relevant to this work, particularly in Section \ref{section:genericity}. We follow in the theory as described in \cite[Ch. 7, 21]{Lee_2013}.
\begin{defin}
    Let $G$ be a smooth manifold admitting the structure of a group where the maps
    $$
    G\times G \rightarrow G, (g_1, g_2)\mapsto g_1g_2
    $$
    $$
    G \rightarrow G, g_1 \mapsto g_1^{-1}
    $$
    are smooth. Then we call $G$ a (real) \textit{Lie group.}
\end{defin}
What is relevant to us of course is to describe the group action of a Lie group on a smooth manifold, and the relationship between its structure and the manifold on which it acts. We shall start by definition the group action of a Lie group:
\begin{defin}
    \par Let $G$ be a Lie group, and $M$ a smooth manifold. A (left) Lie group action of $G$ on $M$ is a group action in which the map $\theta_g: M\rightarrow M$ defined by $x\mapsto g\cdot x$ is a smooth map for all $g\in G$. This action is free if $g\cdot x= x$ for some $x\in M$ implies $g$ is identity element in $G$, and is transitive if for all $x, y\in M$ there exists some $g\in G$ such that $g\cdot x = y$. Finally, if the action is both free and transitive, we call it simply transitive. 
\end{defin}
An important map which relates the structure of a Lie group to a manifold it acts on is the orbit map on a point $x\in M$ given by
\begin{equation}\label{eqn:orbitmap}
    g \mapsto g\cdot x
\end{equation}
If $G$ acts on $M$ transitively, we call $M$ a \textit{homogeneous $G$-space}, and this map is surjective. Moreover, we are given the following result:
\begin{lemma}[\cite{Lee_2013}, Th. 21.18]\label{lemma:LeeLemmaGSpace}
    Let $M$ be a homogeneous $G$-space. Then we get an (equivariant) diffeomorphism $G/\Stab_{G}x\simeq M$ given by (\ref{eqn:orbitmap})
\end{lemma}
\section{Stabilizers of Parameter Space}\label{section:stabilizers}

\par The goal of this section is to classify the possible stabilizers of a parameter $\theta\in \Omega_{(m, n, k)}$ under the group action of $H_n$, which we shall denote $\Stab_{H_n}\theta$. Let $\textup{diag}(\lambda_1, \ldots, \lambda_n)$ denote the diagonal matrix with $\lambda_i$ as the $i$-th diagonal entry, and denote $\textup{diag}(\lambda_{i_1}, \ldots, \lambda_{i_m})$ the matrix with $\lambda_{i_j}$ on the $i_j$-th entry and 1 otherwise. Our main result of the section is the following:
\begin{proposition}\label{thm:stabtheoremmnk}
    Let $\theta\in \Omega_{(m, n, k)}$ be of the form
    $$
    \theta = \left(\begin{bmatrix}
        v_1 & \ldots & v_n
    \end{bmatrix}, \begin{bmatrix}
        a_1^T \\ \vdots \\ a_n^T
    \end{bmatrix}, \begin{bmatrix}
        b_1 \\ \vdots \\ b_n
    \end{bmatrix}, c\right)
    $$
    where $v_i\in\mathbb{R}^{k}, a_i\in\mathbb{R}^m, b_i, c\in\mathbb{R}^k$. Let $\#\{i: (v_i, a_i, b_i)= (0, 0, 0)\}=n'$. Then 
    $$
    \Stab_{H_n} \theta = \langle \mathcal{S}\rangle \times H_{n'}
    $$
    where $\mathcal{S}$ is a finite set with elements of the form $\textup{diag}(\lambda_i, \lambda_j) (i, j)$, where $\lambda_i\lambda_j=1$, with $\#\mathcal{S} \leq \binom{n-n'}{2}$. In particular, there is a bijection
    $$
    \mathcal{S}\longrightarrow \left\{\{i, j\}: \begin{split}
        i\neq j\\ (v_i, a_i, b_i), (v_j, a_j, b_j)\neq (0,0, 0)\\ \exists \lambda>0, \lambda (v_i, a_i, b_i)=(v_j, a_j, b_j)
    \end{split}\right\}
    $$
\end{proposition}
We first note that $H_n$ acts similarly to $M_{n\times m}(\mathbb{R})$ and $\mathbb{R}^{n}$, while acting trivially on the constant term $\mathbb{R}^{k}$. And so in particular:
\begin{equation}
    \Stab_{H_n}(M, A, b, c) = \Stab_{H_n}(M, \begin{bmatrix}
        A & b \end{bmatrix}, 0, 0)
\end{equation}
Moreover, we note the following:
\begin{lemma}\label{lem:intersectionstab}
    Let us denote $\pi_i: \Omega_{(m,n,k)}\rightarrow \Omega_{(m,n,1)}$ to be the projection as before. Then 
    $$
    \bigcap_{i=1}^{k} \Stab_{G}\pi_i(\theta) = \Stab_{G}\theta
    $$
    for any $G\leq H_n$
\end{lemma}
\begin{proof}
    This follows from the definition of the stabilizer.
\end{proof}
So we may just consider $k=1$, and take all our parameters $\theta\in \Omega_{(m,n,1)}$ to be without bias. To this extent, we shall denote all our parameters with $(v, A)$ in place of $(v, A,0,0)\in \Omega_{(m,n,1)}$ 
\begin{lemma}
    For any $\theta\in \Omega$, 
    $$
    \Stab_{S_n}\theta \simeq \prod S_{\nu_i}
    $$
    where $\sum_i \nu_i \leq n$
\end{lemma}
\begin{proof}
    Let $p=\prod p_i\in \Stab_{S_n}\theta$ such that $p_i$ are disjoint cycles. As they are disjoint permutation actions, it follows immediately that $p_i\in\Stab_{S_n}\theta$. Now if $p$ is a $k$-cycle we see by the definition of permutation that $(v_i, a^T_i) = (v_{p^{-1}(i)}, a^T_{p(i)})$ for all $i$. And so, for any $i, j$ such that $p(i)\neq i, p(j)\neq j, (v_i, a_i^T) = (v_j, a_j^T)$. And so the 2-cycle $(i ,j)$ also stabilizes $\theta$, proving our statement. 
\end{proof}

\begin{lemma}
    For any $\theta\in \Omega_{(m,n,1)}$, 
    $$
    \Stab_{D_n^+}\theta \simeq \mathbb{R}_{>0}^{\mu}
    $$
    where $\mu \leq n$
\end{lemma}
\begin{proof}
    We see that if $\textup{diag}(\lambda_1, \ldots, \lambda_n)$ is in $\Stab_{D_n^+}\theta$ then as 
    $$
    \textup{diag}(\lambda_1, \ldots, \lambda_n) = \prod_{i=1}^n \textup{diag}(\lambda_i)
    $$
    with each $\textup{diag}(\lambda_i)$ acting disjointly, we see that $\textup{diag}(\lambda_i)$ also stabilizes $\theta$. Moreover, we see that $(v_i\lambda_i^{-1}, \lambda_ia_i^T) = (v_i, a_i^T)$ for $\lambda_i\neq 1$ if and only if $(v_i, a_i^T)=0$ and so $(v_i\lambda^{-1}, \lambda a_i^T) = (v_i, a_i^T)$ for all $\lambda>0$, proving our statement.
\end{proof}
\begin{proposition}
    For any $\theta\in \Omega_{(m,n,1)}, \Stab_{H_n}\theta$ has the following properties:
    \begin{enumerate}
        \item Any $h\in\Stab_{H_n}\theta$ can be written of the form $lp$ where $l\in D_n^+$ and $p\in S_n$.
        \item $\Stab_{H_n}\theta$ is generated by elements of the form $l'p'$ where $p'$ is a $k$-cycle and $l' = \textup{diag}(\lambda_1, \ldots, \lambda_n)\in D_n^+$ is a diagonal matrix where $\lambda_i\neq 1$ only if $p'(i)\neq i$. In particular, a stabilizing element of the form $l'p'$ acting non-trivially on a pair $(v_i, a_i^T), (v_j, a_j^T)$ exists if and only if they are linearly dependent by some $\lambda>0$. 
        \item Say $p'$ is a $k$-cycle is such that $p'(i)\neq i$ for some $(v_i, a_i^T)\neq 0$. Then there is at most one $l'$ as defined above such that $l'p'\in \Stab_{H_n}\theta$
        \item Finally, in fact, $\Stab_{H_n}\theta$ is generated by elements of the form $\textup{diag}(\lambda_i, \lambda_j) (i, j)$, where $\lambda_i\lambda_j=1$, and $\textup{diag}(\lambda_i)$
    \end{enumerate}
\end{proposition}
\begin{proof}
    \par We see that 
    $$
    (i, j)\textup{diag}(\lambda_1, \ldots, \lambda_i, \lambda_j, \ldots, \lambda_n) = \textup{diag}(\lambda_1, \ldots, \lambda_j, \lambda_i, \ldots, \lambda_n) (i, j)
    $$
    and so $D_n^+S_n = H_n$. This gives us (1).
    \par Now let $l\in D_n^+$ and $p\in S_n$ such that $lp\in \Stab_{H_n}\theta$. We may write $p$ as a product of disjoint cycles $p_1\ldots p_k$ and so correspondingly write $lp = l_1p_1\ldots l_kp_k$ where $l_1\ldots l_k = l$ and $l_i$ acts non-trivially on a coordinate only if $p_i$ acts non-trivially on it. Each of these $l_ip_i$ act disjointly and so $l_ip_i\in\Stab_{H_n}\theta$. 
    \par As such, let is take $p$ to be a $k$-cycle now, and $l = \textup{diag}(\lambda_1, \ldots, \lambda_n)$ to be a diagonal matrix where $\lambda_i\neq 1$ only if $p(i)\neq i$. We then see that if $lp\in\Stab_{H_n}\theta$ then $(v_i\lambda_{i}^{-1}, \lambda_{p(i)}m^T_i) = (v_{p^{-1}(i)}, m^T_{p(i)})$. This implies $v_{p(i)}=\lambda_{p(i)}v_i$ and $\lambda_{p(i)}a^T_i=a_{p(i)}^T$, and so $\lambda_{p(i)}(v_i, a_i^T) = (v_{p(i)}, a_{p(i)}^T)$. This implies that 
    $$
    \lambda_{p^l(i)}'(v_i, a_i^T) = (v_{p^l(i)}, a_{p^l(i)}^T)
    $$
    $$
    \lambda_{p^l(i)}'=\prod_{j=1}^l\lambda_{p^j(i)}.
    $$
    Finally, let $\lambda (v_i, a_i^T) = (v_j, a_j^T)$. Then the action $\textup{diag}(\lambda_i, \lambda_j)(i, j)$ where $\lambda_i = \lambda, \lambda_j = \lambda^{-1}$ stabilizes $\theta$, proving (2)
    \par The statement of (3) comes as a natural consequence of the work above, as assuming otherwise gives us two different $\lambda, \lambda'$ such that $\lambda(v_i, a_i^T) = (v_j, a_j^T) =\lambda'(v_i, a_i^T)$ for some $(v_i, a_i^T)\neq 0$.
    \par To prove (4), we see by (2) that it suffices to prove it for $p$ being a $k$-cycle, and $l = \textup{diag}(\lambda_1, \ldots, \lambda_n)$ being a diagonal matrix where $\lambda_i\neq 1$ only if $p(i)\neq i$. If $lp$ acts non-trivially only on rows of the form $(0, 0)$ then this follows immediately, so we shall assume it acts on at least one row, $(v_i, a_i^T)\neq 0$, non-trivially. We see that
    $$
    (i, p^{k-1}(i))\ldots (i, p(i)) = p
    $$
    and also that 
    $$
    \lambda_{p^l(i)}'(v_i, a_i^T) = (v_{p^l(i)}, a_{p^l(i)}^T)
    $$
    from the work before. Then we see that $\textup{diag}(\lambda^{(l)}_i, \lambda^{(l)}_{p^l(i)})(i, p^l(i))$, where $\lambda^{(l)}_i \lambda^{(l)}_{p^l(i)}=1$ and $\lambda^{(l)}_i  = \lambda'_{p^l(i)}$ stabilizes $\theta$. Moreover, there exists some $l^\dagger\in D_n^+$ such that
    $$
    \textup{diag}(\lambda^{(k-1)}_i, \lambda^{(k-1)}_{p^l(k-1)})(i, p^l(k-1))\ldots \textup{diag}(\lambda^{(1)}_i, \lambda^{(1)}_{p(i)})(i, p(i)) = l^\dagger p
    $$
    also stabilizes $\theta$. This means by (3) that
    $$
    \textup{diag}(\lambda^{(k-1)}_i, \lambda^{(k-1)}_{p^l(k-1)})(i, p^l(k-1))\ldots \textup{diag}(\lambda^{(1)}_i, \lambda^{(1)}_{p(i)})(i, p(i)) = l p
    $$
    proving (4). 
\end{proof}
\begin{corollary}\label{cor:stabmn1statement}
    Let $\theta$ be a parameter with $n'$ rows of the form $(0, 0).$ Then 
    $$
    \Stab_{H_n}\theta = \langle \mathcal{S} \rangle \times H_{n'}
    $$
    where $\mathcal{S}$ is a finite set with elements of the form $\textup{diag}(\lambda_i, \lambda_j) (i, j)$, where $\lambda_i\lambda_j=1$, with $\#\mathcal{S} \leq \binom{n-n'}{2}$. In particular, there is a bijection
    $$
    \mathcal{S}\longrightarrow \left\{\{i, j\}: \begin{split}
        i\neq j\\ (v_i, a_i^T), (v_j, a_j^T)\neq (0,0)\\ \exists \lambda>0, \lambda (v_i, a_i^T)=(v_j, a_j^T)
    \end{split}\}\right\}
    $$
\end{corollary}
\begin{proof}[Proof of Theorem \ref{thm:stabtheoremmnk}]
    This follows from Corollary \ref{cor:stabmn1statement} and Lemma \ref{lem:intersectionstab}: By Corollary \ref{cor:stabmn1statement} we see that
    $$
    \Stab_{H_n}\pi_i(\theta) = \langle \mathcal{S}_i \rangle \times H_{n_{i}'}.
    $$
    And so by Lemma \ref{lem:intersectionstab},
    $$
    \bigcap_{i=1}^{k} \langle \mathcal{S}_i \rangle \times H_{n_{i}'} = \Stab_{H_n}\theta.
    $$
    It follows that 
    $$
    \Stab_{H_n}\theta = \langle \mathcal{S} \rangle \times H_{n'}
    $$
    where $\mathcal{S}$ and $n'$ are given as in the Theorem.
\end{proof}
Finally, we see applying the steps similar as above the following result:
\begin{lemma}\label{lem:StabOmegaPrimeCase}
    Let $A\in M_{n\times m}(\mathbb{R}), b\in \mathbb{R}^n$ be of the form
    $$
    A = \begin{bmatrix}
        a_1^T \\ \vdots \\ a_n^T
    \end{bmatrix}, b= \begin{bmatrix}
        b_1 \\ \vdots \\ b_n
    \end{bmatrix}
    $$
Let $\#\{i: (a_i,b_i)= 0\}=n'$. Then 
    $$
    \Stab_{H_n} (A,b) = \langle \mathcal{S}\rangle \times H_{n'}
    $$
    where $\mathcal{S}$ is a finite set with elements of the form $\textup{diag}(\lambda_i, \lambda_j) (i, j)$, where $\lambda_i\lambda_j=1$, with $\#\mathcal{S} \leq \binom{n-n'}{2}$. In particular, there is a bijection
    $$
    \mathcal{S}\longrightarrow \left\{\{i, j\}: \begin{split}
        i\neq j\\ (a_i, b_i), (a_j, b_j)\neq (0,0)\\ \exists \lambda>0, \lambda (a_i, b_i)=(a_j, b_j)
    \end{split}\right\}
    $$
\end{lemma}
\begin{corollary}\label{cor:StabilizerAbAndFibre}
    Let $\theta = (M, A, b, c)$ be a parameter such that $\Stab_{H_n}(A,b)\neq 1$. Then $\rho_{(m,n,k)}^{-1}(\theta)\not\simeq H_n$
\end{corollary}
\begin{proof}
    By our hypothesis and Lemma \ref{lem:StabOmegaPrimeCase}, we know that there exist $i, j$ such that $(a_i, b_i)=(0,0)$ or there exists some $\lambda>0$ such that $\lambda(a_i,b_i)=(a_j, b_j)$. This gives us two cases to consider:
    \par When $(a_i,b_i)=(0,0)$, we see that $\rho_{(m,n,k)}(M, A, b, c) = \rho_{(m,n,k)}(M^\dagger, A, b, c)$ where 
    $$
    M^\dagger = \begin{bmatrix}
        v_1^\dagger & \ldots & v_n^\dagger
    \end{bmatrix}, v_l^\dagger = \begin{cases}
        v_l & l\neq i \\ 0  & l=i
    \end{cases}.
    $$
    So $\rho_{(m,n,k)}^{-1}(\theta)\not\simeq H_n$.
    \par When there exists some $\lambda>0$ such that $\lambda(a_i,b_i)=(a_j, b_j)$, we see that $\rho_{(m,n,k)}(M, A, b, c) = \rho_{(m,n,k)}(M', A, b, c)$ where 
    $$
    M' = \begin{bmatrix}
        v_1' & \ldots & v_n'
    \end{bmatrix}, v_l' = \begin{cases}
        v_l & l\neq i,j \\ 0  & l=i \\ v_l + \lambda v_i & l = j.
    \end{cases}
    $$
    So $\rho_{(m,n,k)}^{-1}(\theta)\not\simeq H_n$.
\end{proof}
\section{Structure of $\im \rho_{(m ,n, k)}$}\label{section:mn1case}
\par In order to derive some statement on the structure of fibres, diffeomorphic to $H_n$, we shall first describe the structure of $\im\rho_{(m,n,k)}$ as a quotient space of parameter space $\Omega_{(m,n,k)}$. To this extent, it is expedient to first understand the structure of $\im\rho_{(m,n,1)}$ as a quotient space of $\Omega_{(m,n,1)}$ in order to build up to the construction of $\im\rho_{(m,n,k)}$. 
\par As such, let us consider functions in the architecture $\mathbb{R}^m\rightarrow \mathbb{R}^n\rightarrow \mathbb{R}$. We see that for any such $f$, it is of the form
$$
f(x) = v^T\sigma (Ax+b) + C,
$$
for some $A\in M_{m\times n}(\mathbb{R}), v, b\in \mathbb{R}^n,$ and $C\in\mathbb{R}$. We know that there is a diagonal $H_n$ action which preserves the function and so 
$$
f(x) = v^Th^{-1}\sigma h(Ax+b) + C
$$
for all $h\in H_n$. One might hope to regard the space $H_n\backslash \Omega_{(m, n, 1)}$, being the equivalence classes of $\Omega_{(m, n, 1)}$ under the $H_n$-action prescribed above, as the space of functions $f\in C^0(\mathbb{R}^m, \mathbb{R})$ realizable via ReLU on 2 layers: that is to say for any $f\in \im \rho_{(m, n, 1)}$, there is a unique class $\mathcal{C}\in H_n\backslash \Omega_{(n, m, 1)}$ such that $\rho_{(n, m, 1)}(\theta)=f$ for $\theta\in \mathcal{C}$. Indeed this is not true by considering the following obvious case: 
\begin{lemma}\label{lem:weakcounterexample}
    Let 
    $$
    (v, A, b, C) = \left(\begin{bmatrix}
        v_1 \\ \vdots \\ v_n
    \end{bmatrix}, \begin{bmatrix}
        a_1^T \\ \vdots \\ a_n^T
    \end{bmatrix}, \begin{bmatrix}
        b_1 \\ \vdots \\ b_n
    \end{bmatrix}, C\right).
    $$
    such that $v_i=0$ for some $i$. We see that 
    $$
    \rho_{(m,n,1)}(v, A, b, C) = \rho_{(m,n,1)}(v, A_{w}, b_\lambda, C)
    $$
    for all $w\in\mathbb{R}^m, \lambda\in \mathbb{R}$ where 
    $$
    A_w = \begin{bmatrix}
        a_1^T \\ \vdots  \\ a_{i-1}^T \\ w \\ a_{i+1}^T\\ \vdots \\ a_n^T
    \end{bmatrix}, b_\lambda = \begin{bmatrix}
        b_1 \\ \vdots  \\ b_{i-1} \\ \lambda \\ b_{i+1}\\ \vdots \\ b_n
    \end{bmatrix}.
    $$
\end{lemma}
Moreover, even if we require $v_i\neq 0$ for all $i$ in our parameter, we still ourselves short of our hope as we can find the following two relations by direct calculation\footnote{It is in fact easy to see that there more than these two relations. However, for the sake of narrative we shall forego their mention for now.}:
\begin{lemma}\label{lem:strongcounterexample}
    \begin{enumerate}
        \item Suppose there exists some $\lambda>0$ such that $\lambda (a_j, b_j) = (a_i, b_i)$ for some $i\neq j$. Then $\rho_{(m,n,1)}(v, A, b, C) = \rho_{(m, n, 1)}(v', A, b, C)$ where 
        $$
        v'_l = \begin{cases} v_l & l\neq i, j\\ v_i + \lambda v_j & l = i\\ 0 & l=j\end{cases}
        $$
        \item Suppose $a_i=0$ for some $i$. Then $\rho_{(m,n,1)}(v, A, b, C) = \rho_{(m, n, 1)}(v^\dagger, A, b, C+v_i\sigma(b_i))$ where 
        $$
        v^\dagger_l = \begin{cases} v_l & l\neq i\\ 0 & l = i\end{cases}
        $$
    \end{enumerate}
\end{lemma}
It then follows that the space $\im \rho_{(n, m,1)}$ is a quotient of the space
$$
    H_n \backslash \Omega_{(n, m, 1)}/\sim
$$
where the equivalence $\sim$ is generated\footnote{As this equivalence acts simultaneously on the space $\Omega_{(n, m, 1)}$ may rely on the permutation equivalence given by the $H_n$-action to extend the equivalence to all cases covered by Lemmas \ref{lem:weakcounterexample} and \ref{lem:strongcounterexample}} by 
$$
    \left(\begin{bmatrix}
        0 \\ v_2 \\ v_3\\ \vdots
    \end{bmatrix}, \begin{bmatrix}
        a_1^T \\ a_1^T \\ a_3^T\\ \vdots
    \end{bmatrix}, \begin{bmatrix}
        b_1 \\ b_1 \\ b_3\\ \vdots
    \end{bmatrix}, C\right) \sim \left(\begin{bmatrix}
        0 \\ v_2 \\ v_3 \\  \vdots
    \end{bmatrix}, \begin{bmatrix}
        0 \\ a_1^T \\ a_3^T\\ \vdots
    \end{bmatrix}, \begin{bmatrix}
        0 \\ b_1 \\ b_3\\ \vdots
    \end{bmatrix}, C\right) 
    $$
    $$
    \left(\begin{bmatrix}
        v_1 \\ v_2 \\ v_3\\ \vdots
    \end{bmatrix}, \begin{bmatrix}
        a_1^T \\ a_1^T \\ a_3^T\\ \vdots
    \end{bmatrix}, \begin{bmatrix}
        b_1 \\ b_1 \\ b_3\\ \vdots
    \end{bmatrix}, C\right) \sim \left(\begin{bmatrix}
        v_1 +v_2 \\ 0 \\ v_3 \\  \vdots
    \end{bmatrix}, \begin{bmatrix}
        a_1^T \\ a_1^T \\ a_3^T\\ \vdots
    \end{bmatrix}, \begin{bmatrix}
        b_1 \\ b_1 \\ b_3\\ \vdots
    \end{bmatrix}, C\right) 
    $$
    $$
    \left(\begin{bmatrix}
        v_1 \\ v_2 \\  \vdots
    \end{bmatrix}, \begin{bmatrix}
        0 \\ a_2^T \\ \vdots
    \end{bmatrix}, \begin{bmatrix}
        b_1 \\ b_2 \\ \vdots
    \end{bmatrix}, C\right) \sim \left(\begin{bmatrix}
        0 \\ v_2 \\  \vdots
    \end{bmatrix}, \begin{bmatrix}
        0 \\ a_2^T \\ \vdots
    \end{bmatrix}, \begin{bmatrix}
        0 \\ b_2 \\  \vdots
    \end{bmatrix}, C+v_1\sigma(b_1)\right)
    $$
While this space, as we shall learn, does not have a one-to-one correspondence with $\im\rho_{(m, n, 1)}$, it gives us enough information to define a strong enough canonical form, which we shall call the \textit{minimal form}, to understand the structure of $\im \rho_{(n, m, 1)}$:
\begin{defin}\label{def:minimalform}
    \par For a given parameter $\theta =(v, A, b, C)$, consider it's equivalence class, denoted as $\mathcal{C}_{\theta}$, in 
    $$
    H_n \backslash \Omega_{(n, m, 1)}/\sim
    $$
    where the equivalence $\sim$ is generated by
    \begin{enumerate}
        \item 
    $$
    \left(\begin{bmatrix}
        0 \\ v_2 \\ v_3\\ \vdots
    \end{bmatrix}, \begin{bmatrix}
        a_1^T \\ a_1^T \\ a_3^T\\ \vdots
    \end{bmatrix}, \begin{bmatrix}
        b_1 \\ b_1 \\ b_3\\ \vdots
    \end{bmatrix}, C\right) \sim \left(\begin{bmatrix}
        0 \\ v_2 \\ v_3 \\  \vdots
    \end{bmatrix}, \begin{bmatrix}
        0 \\ a_1^T \\ a_3^T\\ \vdots
    \end{bmatrix}, \begin{bmatrix}
        0 \\ b_1 \\ b_3\\ \vdots
    \end{bmatrix}, C\right) 
    $$
        \item 
    $$
    \left(\begin{bmatrix}
        v_1 \\ v_2 \\ v_3\\ \vdots
    \end{bmatrix}, \begin{bmatrix}
        a_1^T \\ a_1^T \\ a_3^T\\ \vdots
    \end{bmatrix}, \begin{bmatrix}
        b_1 \\ b_1 \\ b_3\\ \vdots
    \end{bmatrix}, C\right) \sim \left(\begin{bmatrix}
        v_1 +v_2 \\ 0 \\ v_3 \\  \vdots
    \end{bmatrix}, \begin{bmatrix}
        a_1^T \\ a_1^T \\ a_3^T\\ \vdots
    \end{bmatrix}, \begin{bmatrix}
        b_1 \\ b_1 \\ b_3\\ \vdots
    \end{bmatrix}, C\right) 
    $$
        \item 
    $$
    \left(\begin{bmatrix}
        v_1 \\ v_2 \\  \vdots
    \end{bmatrix}, \begin{bmatrix}
        0 \\ a_2^T \\ \vdots
    \end{bmatrix}, \begin{bmatrix}
        b_1 \\ b_2 \\ \vdots
    \end{bmatrix}, C\right) \sim \left(\begin{bmatrix}
        0 \\ v_2 \\  \vdots
    \end{bmatrix}, \begin{bmatrix}
        0 \\ a_2^T \\ \vdots
    \end{bmatrix}, \begin{bmatrix}
        0 \\ b_2 \\  \vdots
    \end{bmatrix}, C+v_1\sigma(b_1)\right)
    $$
    \end{enumerate}
    \par Now, let $f$ be any function such that $f(1)> f(-1)> f(0)$ and consider the following subset of $\mathcal{C}_{\theta}$: 
    $$
    \left\{(u, A', b', C') \in \mathcal{C}_{\theta} :
    \begin{split}
        u\in \{0, \pm 1\}^n,\\ 
        u_i = 0\Leftrightarrow (a_i', b_i')=(0, 0),\\
        f(u_i)\geq f(u_{i+1}),\\
        f(u_i)=f(u_{i+1})\Rightarrow (a_i, b_i)\geq_l (a_{i+1}, b_{i+1})
    \end{split}\right\}
    $$
    where $>_l$ denotes the lexical ordering. This set is endowed with an ordering by taking $\langle u, u\rangle$, and we shall denote the minimum of this set with respect to this ordering as $\theta_{min}=(v_{min}, A_{min}, b_{min}, C_{min})$, calling it the minimal form of $(v, A, b, C)$. Moreover, $\theta_{min}=\theta'_{min}$ if and only if $\mathcal{C}_{\theta}= \mathcal{C}_{\theta'}$. Finally, if $\theta=\theta_{min}$ then we shall call $\theta$ minimal.
\end{defin}
\begin{remark}\label{remark:generalizing}
    The equivalence $\sim$ does not represent a set of functional symmetries unique to ReLU but rather is applicable almost universally to shallow neural networks. In place of $\sigma$, let us take $f=(f_1, \ldots, f_n)$ to be a function where $f_i=f_j$ for all choices of $i, j$, and consider a general $k$. For any parameter
    $$
    \theta = (M, A, b, c) = \left(\begin{bmatrix}
        v_1 & \ldots & v_n
    \end{bmatrix}, \begin{bmatrix}
        a_1^T \\ \vdots \\ a_n^T
    \end{bmatrix}, \begin{bmatrix}
        b_1 \\ \vdots \\ b_n
    \end{bmatrix}, c\right)
    $$
    in $\Omega_{(m,n,k)}$, we may define a realization with respect to $f$, denoted as $\rho_{(m,n,k)}^{(f)}$, like so:
    $$
    \rho_{(m,n,k)}^{(f)}(\theta) = Mf(Ax+b)+c.
    $$
    Moreover, there is an $S_n$ action on $\Omega_{(m,n,k)}$ given by
    $$
    s\cdot \theta = (Ms^{-1}, sA, sb, c), s\in S_n
    $$
    such that $\rho_{(m,n,k)}^{(f)}$ is an $S_n$-invariant map. Then $\im\rho_{(m,n,k)}^{(f)}$ is a quotient space of 
    $$
    S_n\backslash \Omega_{(m,n,k)}/\sim_f
    $$
    where the equivalence $\sim_f$ is generated by
    $$
    \left(\begin{bmatrix}
        0 & v_2 & v_3 & \ldots
    \end{bmatrix}, \begin{bmatrix}
        a_1^T \\ a_1^T \\ a_3^T\\ \vdots
    \end{bmatrix}, \begin{bmatrix}
        b_1 \\ b_1 \\ b_3\\ \vdots
    \end{bmatrix}, c\right) \sim \left(\begin{bmatrix}
        0 & v_2 & v_3 & \ldots
    \end{bmatrix}, \begin{bmatrix}
        0 \\ a_1^T \\ a_3^T\\ \vdots
    \end{bmatrix}, \begin{bmatrix}
        0 \\ b_1 \\ b_3\\ \vdots
    \end{bmatrix}, c\right) 
    $$
    $$
    \left(\begin{bmatrix}
        v_1 & v_2 & v_3 & \ldots
    \end{bmatrix}, \begin{bmatrix}
        a_1^T \\ a_1^T \\ a_3^T\\ \vdots
    \end{bmatrix}, \begin{bmatrix}
        b_1 \\ b_1 \\ b_3\\ \vdots
    \end{bmatrix}, c\right) \sim \left(\begin{bmatrix}
        v_1 +v_2 & 0 & v_3 & \ldots
    \end{bmatrix}, \begin{bmatrix}
        a_1^T \\ a_1^T \\ a_3^T\\ \vdots
    \end{bmatrix}, \begin{bmatrix}
        b_1 \\ b_1 \\ b_3\\ \vdots
    \end{bmatrix}, c\right) 
    $$
    $$
    \left(\begin{bmatrix}
        v_1 & v_2 &  \ldots
    \end{bmatrix}, \begin{bmatrix}
        0 \\ a_2^T \\ \vdots
    \end{bmatrix}, \begin{bmatrix}
        b_1 \\ b_2 \\ \vdots
    \end{bmatrix}, c\right) \sim \left(\begin{bmatrix}
        0 & v_2 & \ldots
    \end{bmatrix}, \begin{bmatrix}
        0 \\ a_2^T \\ \vdots
    \end{bmatrix}, \begin{bmatrix}
        0 \\ b_2 \\  \vdots
    \end{bmatrix}, c+f_1(b_1)v_1\right).
    $$
    This forms a general theory of quotients for neural networks with activation function of the form of $f$. 
\end{remark}
\begin{defin}\label{def:0factor}
    We shall say a parameter $(v, A, b, C)$ has 0-factors of rank $k$ if $k=\#\{e_i\mid \langle v_{min}, e_i\rangle =0\}$. Likewise if a parameter $(v, A, b, C)$ has 0-factors of rank 0, we shall say it has no 0-factors.
\end{defin}
\begin{example} Given below are some examples of minimal forms of parameters
\begin{center}
    \begin{tabular}{|c|c|c|}
        \hline
        Parameter & Minimal Form & 0-factor Rank \\
        \hline
        $\left(\begin{bmatrix} 0 \\ 0 \\ 0\end{bmatrix}, A, b, C \right)$ & $\left(\begin{bmatrix} 0 \\ 0 \\ 0\end{bmatrix}, \begin{bmatrix} 0 & 0 & 0 \\ 0 & 0 & 0 \\ 0 & 0 & 0\end{bmatrix}, \begin{bmatrix} 0 \\ 0 \\ 0\end{bmatrix}, C\right)$ & 3\\
        \hline 
        $\left(\begin{bmatrix} 2 \\ 3 \\ -1\end{bmatrix}, \begin{bmatrix} 4 & 4 & 4 \\ 1 & 1 & 1 \\ 5 & 2 & 1\end{bmatrix}, \begin{bmatrix} 0 \\ 0 \\ 0\end{bmatrix}, C\right)$ & $\left(\begin{bmatrix} 1 \\ -1 \\ 0\end{bmatrix}, \begin{bmatrix} 11 & 11 & 11 \\ 5 & 2 & 1 \\ 0 & 0 & 0\end{bmatrix}, \begin{bmatrix} 0 \\ 0 \\ 0\end{bmatrix}, C\right)$ & 1\\
        \hline
        $\left(\begin{bmatrix} 1 \\ 1 \\ -1\end{bmatrix}, \begin{bmatrix} 2 & 4 & 0 \\ 1 & 1 & 3 \\ 0 & 0 & 0\end{bmatrix}, \begin{bmatrix} 7 \\ 2 \\ 3\end{bmatrix}, 7\right)$ & $\left(\begin{bmatrix} 1 \\ 1 \\ 0\end{bmatrix}, \begin{bmatrix} 2 & 4 & 0 \\ 1 & 1 & 3 \\ 0 & 0 & 0\end{bmatrix}, \begin{bmatrix} 7 \\ 2 \\ 0\end{bmatrix}, 4\right)$ & 1\\
        \hline
        $\left(\begin{bmatrix} 1 \\ 1 \\ -1\end{bmatrix}, \begin{bmatrix} 0 & 36 & 0 \\ 1 & 7 & 3 \\ 10 & 5 & 15\end{bmatrix}, \begin{bmatrix} 63 \\ 2 \\ 45\end{bmatrix}, 8\right)$ & $\left(\begin{bmatrix} 9 \\ 1 \\ -5\end{bmatrix}, \begin{bmatrix} 0 & 4 & 0 \\ 1 & 7 & 3 \\ 2 & 1 & 3\end{bmatrix}, \begin{bmatrix} 7 \\ 2 \\ 9\end{bmatrix}, 8\right)$ & 0\\
        \hline
    \end{tabular}
\end{center}
\end{example}

\begin{lemma}\label{lem:minred}
    If $v_1^T\sigma (A_1+b_1) + C_1 = v^T_2\sigma (A_2+b_2) + C_2$ then 
    $v_{1, min}^T\sigma (A_{1, min}+b_{1, min}) + C_{1, min} = v^T_{2, min}\sigma (A_{2, min}+b_{2, min}) + C_{2, min}$.
\end{lemma}
\begin{definlem}\label{deflem:reduced}
    Given a parameter $\theta=(v, A, b, C)$ with zero factors of rank $k<n$, we see that there exist unique vectors $v_{r}\in \mathbb{R}^{n-k}, b_r\in\mathbb{R}^{m-k}$, and a unique matrix $A_r\in M_{(n-k)\times m}(\mathbb{R})$ such that $\theta_r=(v_r, A_r, b_r, C_r)$ is minimal with no 0-factors and $v^T\sigma A = v_r^T\sigma A_r$. We shall call $\theta_r$ the 0-factor reduction of $\theta$ in codimension $k$. Moreover, if $\theta$ has zero factors of rank $n$, by convention we shall say $\theta_r = (0, 0, 0, C_r)\in \mathbb{R}^4$. Finally, $\theta=\theta_r$ if and only if $\theta$ is minimal with no zero factors, and for $\theta_i\in\Omega_{(m,n,1)}, \theta_{1, r}=\theta_{2,r}$ if and only if $\theta_{1, min}=\theta_{2, min}$. 
\end{definlem}

\begin{defin}\label{def:semilinindep}
    Any two vectors $v, w$ are linearly semi-independent if there exists no $\lambda\geq 0$ such that $\lambda v = w$, and linearly semi-dependent if there is such a $\lambda\geq 0$. Now let $v_1, \ldots, v_n$ be an indexed set of vectors. We call them pairwise linearly semi-independent if all distinct pairs $v_i, v_j$ are linearly semi-independent. Likewise, this set is called pairwise linearly independent if instead of $\lambda\geq 0$ we allows $\lambda$ to be any real number. 
\end{defin}
\begin{lemma}
    Let $(v, A, b, C)$ be minimal with no $0$-factors. Then the vectors formed by the rows of $A$ and the corresponding coordinates of $b$ are pairwise linearly semi-independent. 
\end{lemma}
\begin{proof}
    This follows from the minimality of $(v, A, b, C)$
\end{proof}
\begin{proposition}\label{thm:nontrivialzeroparam}
    Let $(v, A, b, C)$ be minimal with no $0$-factors, and $A\neq 0$. Then $v^T\sigma (Ax+b) + C = Tx + D$ for some linear transformation $T:\mathbb{R}^n\rightarrow \mathbb{R}$ and $D\in\mathbb{R}$ if and only if $v = e_1 + \ldots + e_l - e_{l+1} - \ldots - e_{2l}$ and 
    $$
    A = \begin{bmatrix}
        a_1^T\\
        \vdots\\
        a_l^T\\
        -a_{1}^T\\
        \vdots\\
        -a_{l}^T
    \end{bmatrix}, b = \begin{bmatrix}
        b_1\\
        \vdots\\
        b_l\\
        -b_1\\
        \vdots\\
        -b_l
    \end{bmatrix}
    $$
    with
    $$
    \sum_{i=1}^l \langle a_i, x\rangle + b_i + C = Tx + D
    $$
\end{proposition}
\begin{proof}
    \par We see that to show this is equivalent to showing that
    $$
    \sum_{i=1}^k \sigma (\langle v_i, x\rangle + \alpha_i) = \sum_{i=1}^l \sigma (\langle w_i, x\rangle + \beta_i) + Tx+D
    $$
    for $(v_1, \alpha_1), \ldots, (v_k, \alpha_k)$, $(w_1, \beta_1), \ldots, (w_l, \beta_l)$ pairwise linearly semi-independent can only occur when $k=l$ and $(v_i, \alpha_i) = -(w_i, \beta_i)$ for all $i$ (by reordering the terms.) Let us say either $k> l$ with $v_{l+1}\neq 0$ (this follows from minimality) or $(v_i, \alpha_i) \neq -(w_j, \beta_j)$ for all $j$ with $v_i\neq 0$. 
    \par Our assumption then gives that there is a vector $x_1$ such that $\alpha_i/\langle v_i, x_1\rangle\neq \beta_j/\langle w_j, x_1\rangle$ for all choices of $j$. Without loss of generality let us say $(v_1, \alpha_1)$ satisfies this condition. And so restricting our function to $tx_1$ for $t\in\mathbb{R}$ we reduce to one of two cases, denoting $c_1=\langle v_1, x_1\rangle$:
    $$
    \lambda_1 \sigma (c_1 t + \alpha_1) + \sum_{i=2}^{k'} \lambda_i\sigma (c_i t + \alpha_i') = \sum_{i=1}^{l'}\sigma(d_i t + \beta_i') + \Lambda t + D
    $$
    or
    $$
    \lambda_1 \sigma (c_1 t + \alpha_1) + \lambda_1' \sigma(-c_1 t - \alpha_1) + \sum_{i=2}^{k'} \lambda_i\sigma (c_i t + \alpha_i') = \sum_{i=1}^{l'}\mu_i\sigma(d_i t + \beta_i') + \Lambda t + D
    $$
    for $\lambda_i, \mu_j, \lambda_1'>0$ and $(c_i, \alpha_i'), (d_i, \beta_i')$ pairwise linearly semi-independent, and $(c_1, \alpha_1)$ pairiwse linearly independent with $(c_i, \alpha_i'), (d_i, \beta_i'), i\neq 1$. We see the equations above is equivalent to writing
    $$
    \lambda_1 \sigma (c_1 t + \alpha_1) = \sum_{i=1}^{l'}\sigma(d_i t + \beta_i') - \sum_{i=2}^{k'} \lambda_i\sigma (c_i t + \alpha_i') + \Lambda t + D
    $$
    or
    $$
    \lambda_1 \sigma (c_1 t + \alpha_1) + \lambda_1' \sigma(-c_1 t - \alpha_1) = \sum_{i=1}^{l'}\sigma(d_i t + \beta_i') - \sum_{i=2}^{k'} \lambda_i\sigma (c_i t + \alpha_i') + \Lambda t + D
    $$
    \par By assumption, we see that $\alpha_1/c_1\neq \alpha_i'/c_i, \beta'_j/d_j, i\neq 1$ and so we can pick some neighbourhood $U$ of $-\alpha_1/c_1$ such that for all $t\in U$,
    $$
    \sum_{i=1}^{l'}\sigma(d_i t + \beta_i') - \sum_{i=2}^{k'} \lambda_i\sigma (c_i t + \alpha_i') + \Lambda t + D
    $$
     is a linear function, and so in particular is differentiable at $-\alpha_1/c_1$. However we see clearly that neither $\lambda_1 \sigma (c_1 t + \alpha_1)$ nor $\lambda_1 \sigma (c_1 t + \alpha_1) + \lambda_1' \sigma(-c_1 t - \alpha_1)$ is differentiable at $-\alpha_1/c_1$, giving us a contradiction. 
\end{proof}
\begin{corollary}\label{cor:minranklinind}
    Let $(v, A, b, C)$ be an arbitrary parameter, and denote
    $$
    v = \begin{bmatrix}
        v_1 \\ \vdots\\ v_n
    \end{bmatrix}, A = \begin{bmatrix}
        a_1^T\\ \vdots\\ a_n^T
    \end{bmatrix}, b = \begin{bmatrix}
        b_1 \\ \vdots\\ b_n
    \end{bmatrix}
    $$
    with $v_i\neq 0$ and the property that there exists no $\lambda<0$ such that $\lambda (a_i, b_i) = (a_j, b_j)$. Then if $(v', A', b', C')=(v, A, b, C)$ and $(v, A, b, C)$ has zero factors of rank $l$, then $(v', A', b', C')$ has zero factors of at most rank $l$. 
\end{corollary}
\begin{proof}
    \par By considering the 0-factor reductions of $(v, A, b, C)$ and $ (v', A', b', C')$ we reduce this problem: For $\lambda_i, \lambda_j'\in \{\pm 1\}, a_i, a_j'\neq 0 , (a_1, b_1), \ldots, (a_k, b_k)$ pairwise linearly independent and $(a_1', b_1'), \ldots, (a_{k'}', b'_{k'})$ pairwise linearly semi-independent,
    $$
    \sum_{j=1}^{k'} \lambda_j'\sigma (\langle a_j', x\rangle + b_j') + C' = \sum_{i=1}^{k} \lambda_i\sigma (\langle a_i, x\rangle + b_i) + C
    $$
    is impossible if $k'<k$. Moreover, if $(a_j', b_j')=(a_i, b_i)$ we may add both sides by $-\lambda_j'\sigma(\langle a_i', x\rangle+b_i')$ and reduce, by which we may assume without loss of generality that  $(a_1, b_1), \ldots, (a_k, b_k), (a_1', b_1'), \ldots, (a_{k'}', b'_{k'})$ are pairwise linearly semi-independent. This is equivalent to producing a minimal parameter $(v^\dagger, A^\dagger, b^\dagger, C^\dagger)$ with no $0$-factors in $M_{1\times (k+k')}(\mathbb{R})\times M_{(k+k')\times m}(\mathbb{R})\times\mathbb{R}^{(k+k')+1}$ such that $(a_1^\dagger, b_1^\dagger), \ldots, (a_{k+k'}^\dagger, b_{k+k'}^\dagger)$ are pairwise linearly independent and 
    $$
    v^{\dagger T}\sigma (A^\dagger x +b^\dagger)+ C^\dagger = 0.
    $$
    However as $k>k'$, this is impossible by Proposition \ref{thm:nontrivialzeroparam}. 
\end{proof}
\begin{corollary}\label{cor:semidependenttermscancelling}
    Let $(v, A, b, C)$ be a minimal parameter with no 0-factors such that $(a_1, b_1), \ldots, (a_n, b_n)$ are pairwise linearly independent. Suppose that there exists a minimal parameter $(v', A', b', C')$ such that $v'^T\sigma (A' x + b')+C' = v^T\sigma (Ax + b) + C$, and that there exists some $i$ such that $v_i'\neq 0$ and there exists a $\lambda>0$ with $(a_i', b_i')=\lambda (a_1, b_1)$. Then $v_1-\lambda v_i'=0$. 
\end{corollary}
\begin{proof}
    Supposing towards contradiction by letting $v_1-\lambda v_i'\neq 0$ , subtracting both sides by $\lambda v_i'\sigma(\langle a_i, x\rangle + b_i)$ gives a counterexample to Corollary \ref{cor:minranklinind}, and so cannot be true.
\end{proof}
\begin{defin}
    For parameters $\theta_i=(v_i, A_i, b_i, C_i)\in \Omega_{(m, n_i, 1)}$ we define $\theta_1\ominus \theta_2\in \Omega_{(m, n_1+n_2, 1)}$ to be
    $$
    \theta_1\ominus \theta_2 = \left(\begin{bmatrix} v_1 \\ -v_2 \end{bmatrix}, \begin{bmatrix} A_1 \\ A_2 \end{bmatrix},\begin{bmatrix} b_1\\ b_2 \end{bmatrix}, C_1-C_2\right)
    $$
\end{defin}
\begin{lemma}\label{lem:0Hnrelation}
    $(\theta_1\ominus \theta_2)_r = (0, 0, 0, 0)$ if and only if $\theta_{1, r}=\theta_{2, r}$ 
\end{lemma}
\begin{proof}
    \par We see that $(\theta_1\ominus \theta_2)_r  = (\theta_{1, r}\ominus \theta_{2, r})_r $ as the 0-factor reduction depends only on the equivalence class of $\theta_1\ominus \theta_2$ in $H_n\backslash \Omega_{(m, n_1+n_2, 1)}/\sim$, with $\theta_1\ominus \theta_2$ and $\theta_{1, min}\ominus \theta_{2, min}$ lying in the same equivalence class. As such, let us assume $\theta_1, \theta_2$ are reduced, and denote
    $$
    \theta_i = \left(\begin{bmatrix}
        v_{i, 1} \\ \vdots \\ v_{i,n_i}
    \end{bmatrix}, \begin{bmatrix}
        a_{i, 1}^T \\ \vdots \\ a_{i,n_i}^T
    \end{bmatrix}, \begin{bmatrix}
        b_{i,1} \\ \vdots \\ b_{i,n_i}
    \end{bmatrix}, C_i\right)
    $$
    We see that the first and third generating equivalence of $\sim$ in the definition of the minimal form act trivially on the parameter $\theta_r\ominus \theta'_r$. Then $n_1=n_2=n$ and for all rows $(a_{1, i}, b_{1, i})$ there exist $(a_{2, j}, b_{2, j})$ and $\lambda_{i}>0$ such that $\lambda_i (a_{1, i}, b_{1, i})= (a_{2, j}, b_{2, j})$. This means that $\theta_1\ominus\theta_2$ are in the same equivalence class of $H_n\backslash \Omega_{(m, 2n, 1)}/\sim$ as
    $$
    \left(\begin{bmatrix}
        v_{1, 1}-\lambda_{1}v_{2, p(1)} \\ \vdots \\ v_{1, n}-\lambda_{n}v_{2, p(n)} \\ 0 \\ \vdots \\ 0
    \end{bmatrix}, \begin{bmatrix}
        a_{1, 1}^T \\ \vdots \\ a_{1,n}^T \\ 0 \\ \vdots \\ 0
    \end{bmatrix}, \begin{bmatrix}
        b_{1,1} \\ \vdots \\ b_{1,n} \\ 0 \\ \vdots \\ 0
    \end{bmatrix}, C_1-C_2\right)
    $$
    for some permutation $p\in S_n$. As we assume $\theta_i$ are reduced, we know that $v_{i, j}=\pm 1$ for all $i, j$, which means that $\lambda_i = 1$ for all $i$. This means that $v_{1, j}=v_{2, p(j)}$ and $(a^T_{1, j}, b_{1, j})= (a^T_{2, p(j)}, b_{2, p(j)})$, and since $\theta_2$ is reduced means that $p=e\in S_n$ and so $\theta_1=\theta_2$ when $\theta_i$ are reduced, proving our statement. 
\end{proof}
\begin{corollary}\label{cor:ominusred0cor}
    For $\theta_1, \theta_2\in H_n\backslash\Omega_{(m,n,1)}/\sim, (\theta_1\ominus \theta_2)_r = (0, 0, 0, 0)$ if and only if $\theta_{1, min}=\theta_{2, min}$ 
\end{corollary}
\begin{theorem}\label{thm:symmetryclassificationk=1}
    The quotient space $\im\rho_{(m, n,1)}$ is given by
    $$
    H_n\backslash\Omega_{(m,n,1)}/\equiv
    $$
    where $\theta_1\equiv\theta_2$ if and only if $\theta_1\sim\theta_2$ or 
    $$
    (\theta_1\ominus\theta_2)_r = \left(\begin{bmatrix}
        1\\
        \vdots\\
        1\\
        -1\\
        \vdots\\
        -1
    \end{bmatrix}
    \begin{bmatrix}
        a_1^T\\
        \vdots\\
        a_l^T\\
        -a_{1}^T\\
        \vdots\\
        -a_{l}^T
    \end{bmatrix}, \begin{bmatrix}
        b_1\\
        \vdots\\
        b_l\\
        -b_1\\
        \vdots\\
        -b_l
    \end{bmatrix}, C
    \right)
    $$
    where $\sum a_i = 0$ and $\sum b_i + C = 0$
\end{theorem}
\begin{proof}
    We see that if $\rho_{(m,n,1)}(\theta_1)=\rho_{(m,n,1)}(\theta_2)$ that $\rho_{(m, 2n, 1)}(\theta_1\ominus\theta_2)=0$. We see by Corollary \ref{cor:ominusred0cor} that $(\theta_1\ominus\theta_2)_r=(0,0,0,0)$ if and only if $\theta_1$ and $\theta_2$ lie in the same equivalence class of $H_n\backslash \Omega_{(m,n,1)}/\sim$. On the other hand, if $(\theta_1\ominus\theta_2)_r\neq(0,0,0,0)$ we see by Proposition \ref{thm:nontrivialzeroparam} that $(\theta_1\ominus\theta_2)_r$ must be of the form
    $$
    (\theta_1\ominus\theta_2)_r = \left(\begin{bmatrix}
        1\\
        \vdots\\
        1\\
        -1\\
        \vdots\\
        -1
    \end{bmatrix}
    \begin{bmatrix}
        a_1^T\\
        \vdots\\
        a_l^T\\
        -a_{1}^T\\
        \vdots\\
        -a_{l}^T
    \end{bmatrix}, \begin{bmatrix}
        b_1\\
        \vdots\\
        b_l\\
        -b_1\\
        \vdots\\
        -b_l
    \end{bmatrix}, C
    \right)
    $$
    where $\sum a_i = 0$ and $\sum b_i + C = 0$, as desired. 
\end{proof}
\begin{theorem}\label{thm:symmetryclassification}
    The quotient space $\im \rho_{(m,n,k)}$ is given by
    $$
    \Omega_{(m,n,k)}/\equiv_k
    $$
    where $\theta_1\equiv_k \theta_2$ if and only if $\pi_i(\theta_1)\equiv \pi_i(\theta_2)\in H_n\backslash\Omega_{(m,n,1)}$ for all projections $\pi_i:\Omega_{(m,2n,k)}\rightarrow \Omega_{(m,2n,1)}$.
\end{theorem}
\begin{proof}
    We see that $\rho_{(m,n,k)}(\theta_1)=\rho_{(m,n,k)}(\theta_2)$ if and only if $\rho_{(m,n,1)}(\pi_i(\theta_1))=\rho_{(m,n,1)}(\pi_i(\theta_2))$ for all $i,$ which is equivalent to the condition that $\pi_i(\theta_1)\equiv \pi_i(\theta_2)$ in $\Omega_{(m,n,1)}$ 
\end{proof}

\section{Proof of Theorem \ref{thm:maintheorem}}\label{section:genericity}
Our goal this section is to prove the following theorem, which essentially gives the statement of Theorem \ref{thm:maintheorem}:
\begin{theorem}\label{thm:ZariskiOpenSet}
    Let $\Omega^{(1)}_{(m, n, k)}\subset \Omega_{(m, n, k)}$ be the set of parameters $\theta$ such that the action of $H_n$ on $\rho_{(m,n,k)}^{-1}(\theta)$ is simply transitive. Then $\Omega^{(1)}_{(m, n, k)}$ contains a Zariski open set.
\end{theorem}

With our previous work in Section \ref{section:mn1case}, we have essentially approached the solution. All that remains is to prove some additional results and explicitly construct a Zariski open set in $\Omega^{(1)}_{(m,n,k)}$. In order to do this, we shall first start by reducing the problem to the $\Omega^{(1)}_{(m,n,1)}$:
\begin{lemma}\label{lem:projzarcont}
    The projection map $\pi_i: \Omega_{(m,n,k)}\rightarrow\Omega_{(m,n,1)}$ is Zariski continuous.
\end{lemma}
\begin{proof}
    Let us denote
    $$
    (M, A, b, c) = \left(\begin{bmatrix}
        m_1 \\ \vdots\\ m_k
    \end{bmatrix}, A, b, \begin{bmatrix}
        c_1 \\ \vdots\\ c_k
    \end{bmatrix}\right)\in \Omega_{(m,n,k)}.
    $$
    Let $f(v, A, b, C):\Omega_{(m,n,1)}\rightarrow\mathbb{R}$ be some polynomial function. We see then that $\pi_i^{-1}(V(f))=V(f_i)$ where $f_i(M, A, b, c) = f(m_i, A, b, c_i)$
\end{proof}
\begin{proposition}\label{prop:genericityreductionprop}
    If $\Omega^{(1)}_{(m, n, 1)}$ contains a Zariski open set then $\Omega^{(1)}_{(m, n, k)}$ contains a Zariski open set.
\end{proposition}
Namely, denoting $\textup{proj}_i:\mathbb{R}^{n}\rightarrow\mathbb{R}$ as the projection to the $i$-th coordinate,
$$
\textup{proj}_i(\rho_{(m,n,k)}(\theta)) = \rho_{(m,n,1)}(\pi_i(\theta))
$$
which in particular gives us that
\begin{equation}\label{eqn:fibrecontainment}
    \pi_i(\rho_{(m,n,k)}^{-1}(\theta))\subseteq \rho_{(m,n, 1)}^{-1}(\pi_i(\theta))
\end{equation}
\par Moreover, recalling the definition of the fibre product, for any given subsets $S_1\subseteq A\times C, S_2\subseteq A\times C$, we see that $S_1\times_{C}S_2$ can be interpreted like so:
$$
\{(a,b,c)\in A\times B\times C: (a, c)\in S_1, (b, c)\in S_2\}.
$$
Given this, we see there is the obvious inclusion, denoting $\Omega' = \{(A, b): (v,A,b,C)\in \Omega_{(m,n,1)}\} = \{(A, b): (v,A,b,C)\in \Omega_{(m,n,k)}\}$:
$$
{\sideset{}{_{\Omega'}}\bigtimes_{i=1}^{k}}\rho_{(m,n,1)}^{-1}(\pi_i(\theta)) \subseteq \rho_{(m,n,k)}^{-1}(\theta).
$$
However by Equation \ref{eqn:fibrecontainment} we see that in fact
\begin{equation}\label{eqn:fibreequality}
    {\sideset{}{_{\Omega'}}\bigtimes_{i=1}^{k}}\rho_{(m,n,1)}^{-1}(\pi_i(\theta)) = \rho_{(m,n,k)}^{-1}(\theta).
\end{equation}
With this, we shall prove this abstract result in category theory applicable to our case:
\begin{lemma}\label{lem:cattheorystatement}
    Let $A, B, C, D$ be objects in some category such that we have the diagram
    \begin{center}
\begin{tikzcd}
D \arrow[rrd, "f_1", bend left] \arrow[rdd, "f_2"', bend right] \arrow[rd, "f", dashed] &                                                &                    \\
                                                                                        & A\times_C B \arrow[d, "e_2"'] \arrow[r, "e_1"] & B \arrow[d, "g_1"] \\
                                                                                        & A \arrow[r, "g_2"]                             & C                 
\end{tikzcd}
    \end{center}
    with $f_1, f_2$ isomorphisms, and $g_1, g_2$ monomorphisms. Then $f$ is an isomorphism.
\end{lemma}
\begin{proof}
    We see that $A\times_C B\simeq D\times_C D$ where 
    \begin{center}
\begin{tikzcd}
D\times_C D \arrow[d, "f_2^{-1}\circ e_2"'] \arrow[r, "f_1^{-1}\circ e_1"] & D \arrow[d, "g_1\circ f_1"] \\
D \arrow[r, "g_2\circ f_2"']                                               & C                          
\end{tikzcd}
    \end{center}
    We see that $g_2\circ f_2 = g_1\circ f_1$ and are monomorphisms and so $D\times_C D \simeq D$, which implies $f_2^{-1}\circ e_2$ is an isomorphism. This therefore means that $f=(f_2^{-1}\circ e_2)^{-1}$ is an isomorphism. 
\end{proof}
\begin{proof}[Proof of Proposition \ref{prop:genericityreductionprop}]
    \par To prove this statement, we see by Lemma \ref{lem:projzarcont} and Equation \ref{eqn:fibreequality} that it suffices to show $\rho_{(m,n,1)}^{-1}(\pi_i(\theta))\simeq H_n$ for all $i$ implies $\rho_{(m,n,k)}^{-1}(\theta)\simeq H_n$. By Lemma \ref{lem:cattheorystatement}, This reduces to showing $\rho_{(m,n,k)}^{-1}(\theta)\simeq H_n$ implies $\rho_{(m,n,k)}^{-1}(\theta)\rightarrow \Omega'$ is injective. Suppose $\rho_{(m,n,k)}^{-1}(\theta)\simeq H_n$ and $\rho_{(m,n,k)}^{-1}(\theta)\rightarrow \Omega'$ is not injective. Then $\Stab_{H_n}(A,b)\neq 1$, which implies by Corollary \ref{cor:StabilizerAbAndFibre} that $\rho_{(m,n,k)}^{-1}(\theta)\not\simeq H_n$, giving us a contradiction. 
\end{proof}
Now that we have reduced this problem to that of $\Omega_{(m,n,1)}$, we shall explicitly construct a Zariski open set $U_\sim$ of $\Omega_{(m,n,1)}^{(1)}$. First, we shall construct a Zariski open set of $\Omega_{(m,n,1)}$ in which the equivalence $\sim$ acts trivially in $H_n\backslash \Omega_{(m,n,1)}$. From this we shall construct an open set of $U_\sim$ in which the equivalence $\equiv$ acts trivially in $H_n\backslash \Omega_{(m,n,1)}$:
\begin{lemma}
    Let
    $$
    u = \begin{bmatrix}
        u_1 \\ \vdots\\ u_n
    \end{bmatrix}, w = \begin{bmatrix}
        w_1 \\ \vdots \\ w_n
    \end{bmatrix}.
    $$
    Then $u$ and $w$ are linearly dependent if and only if 
    $$
    u_i w_j - w_i u_j =0
    $$
    for all $i, j, i\neq j$.
\end{lemma}
\begin{defin}
    For $u\in \mathbb{A}^n, w\in \mathbb{A}^n$, let
    $$
    V^{dep}(u, w) = \bigcap_{i\neq j} V(u_i w_j - w_i u_j) \subset \mathbb{A}^{2n}
    $$
    denote the Zariski closed set of ordered pairs of linearly dependent vectors in $\mathbb{A}^n$. Likewise for $u = \begin{bmatrix}
        u_1 \\ \vdots \\ u_n
    \end{bmatrix}$ let
    $$
    V(u) = V(u_1, \ldots, u_n)
    $$
\end{defin}
\begin{lemma}
    Denote
    $$
    \theta = \left(\begin{bmatrix}
        v_1 \\ \vdots \\ v_n
    \end{bmatrix}, \begin{bmatrix}
        a_1^T \\ \vdots \\ a_n^T
    \end{bmatrix}, \begin{bmatrix}
        b_1 \\ \vdots \\ b_n
    \end{bmatrix}, C\right).
    $$
    Then 
    $$
    U_\sim = (V_1 \cup V_2 \cup V_3)^c,
    $$
    where
    $$
    V_1 = \bigcup_{i=1}^n V(v_i)
    $$
    $$
    V_2 = \bigcup_{i\neq j} V^{dep}((a_i, b_i), (a_j, b_j)).
    $$
    $$
    V_3 = \bigcup_{i}V(a_i)
    $$
\end{lemma}
\begin{proof}
    Let us denote by $\mathcal{S}(i)$ the set of parameters $\theta\in \Omega_{(m,n,1)}$ such that the equivalence (i) as in Definition \ref{def:minimalform} acts nontrivially on $\theta$ in $H_n\backslash \Omega_{(m,n,1)}$. By their definition, we see that $V_i$ contains $S(i)$.
\end{proof}
It follows immediately that the minimal form of any parameter in $U_\sim$ has no $0$-factors. With this in mind we prove the following statement:
\begin{proposition}\label{prop:mn1genericity}
    Let $\theta=(v, A, b, C)$ be an arbitrary parameter, and denote
    $$
    v = \begin{bmatrix}
        v_1 \\ \vdots\\ v_n
    \end{bmatrix}, A = \begin{bmatrix}
        a_1^T\\ \vdots\\ a_n^T
    \end{bmatrix}, b = \begin{bmatrix}
        b_1 \\ \vdots\\ b_n
    \end{bmatrix}.
    $$
    Let us assume the following conditions: $v_i\neq 0, (a_1, b_1), \ldots, (a_n, b_n)$ pairwise linearly independent, and 
    $$
    \sum_{i=1}^n \beta(i)v_ia_i\neq 0.
    $$
    for all functions $\beta: \{1, \ldots, n\}\rightarrow \{0,  \pm1\}, \beta\neq 0$. Then the only other parameters in $\Omega_{(m,n,1)}$ which realize $v^T\sigma(Ax+b)+C$ are of the form $(vh^{-1}, hA, hb, C)$ where $h\in H_n$.
\end{proposition}
\begin{proof}
    We see by our hypothesis that $(v, A, b, C)$ has no 0-factors, and so it follows that
    $$
    \sum_{i=1}^n \beta(i) a_{i, min} \neq 0,
    $$
    for all functions $\beta: \{1, \ldots, n\}\rightarrow \{0,  \pm 1\}, \beta\neq 0$. So we may assume without loss of generality that $(v, A, b, C)$ is minimal with no 0-factors. Now if $\theta'=(v',A', b', C')$ is another minimal parameter such that $v'^T\sigma (A'x+b')+C'=v^T\sigma (Ax+b)+C$, we see by Corollary \ref{cor:minranklinind} that it must have no 0-factors. Moreover, by definition $\rho_{(m,2n,1)}(\theta\ominus\theta')=0$, which implies $\rho_{(m,n,1)}((\theta\ominus\theta')_r)=0$. By Corollary \ref{cor:semidependenttermscancelling}, we see that the rows of $\left(\begin{bmatrix}
        A \\ A'
    \end{bmatrix}_r, \begin{bmatrix}
        b \\ b'
    \end{bmatrix}_r\right)$ are pairwise linearly semi-independent. By Proposition \ref{thm:nontrivialzeroparam} implies that there exists some indices $\{i_1, \ldots, i_{n'}\}$ and a morphism of sets $\gamma : \{1, \ldots, r\}\rightarrow \{\pm 1\}$ such that 
    $$
    (\theta\ominus \theta')_r = \left(\begin{bmatrix}
        1 \\ \vdots \\ 1 \\ -1 \\ \vdots \\ -1
    \end{bmatrix}, \begin{bmatrix}
        \gamma(1)a_{i_1}^T \\ \vdots \\ \gamma(n')a_{i_{n'}}^T \\ -\gamma(1)a_{i_1}^T \\ \vdots \\ -\gamma(n')a_{i_{n'}}^T
    \end{bmatrix}, \begin{bmatrix}
        \gamma(1)b_{i_1} \\ \vdots \\ \gamma(n')b_{i_{n'}} \\ -\gamma(1)b_{i_1} \\ \vdots \\ -\gamma(n')b_{i_{n'}}
    \end{bmatrix}, C-C'\right) \textup{ or } (0,0,0,0).
    $$
    And so
    \begin{align*}
        \rho_{(m,n,1)}((\theta\ominus \theta')_r) &= \sum_{l=1}^{n'} \gamma(l)(\sigma(a_{i_l}^Tx+b_{i_l}) - \sigma(-a_{i_l}^Tx-b_{i_l}))\\
        &= \sum_{l=1}^{n'} \gamma(l)(a_{i_l}^Tx+b_{i_l})\\
        \intertext{This means there exists some $\beta':\{1, \ldots, n\}\rightarrow \{0, \pm 1\}$ such that}
       0 &= \sum_{i=1}^n \beta'(i)(a_{i}^Tx+b_{i})
    \end{align*}
    However by assumption
    $$
    \sum_{i=1}^n \beta(i) a_{i} \neq 0,
    $$
    for all $\beta:\{1, \ldots, n\}\rightarrow \{0, \pm 1\}, \beta\neq 0$ and so $\beta' =0$. This means that there cannot exist any such indices $\{i_1, \ldots, i_{n'}\}$, and so $(\theta\ominus\theta')_r=(0,0,0,0)$. This means by Lemma \ref{lem:0Hnrelation} that $\theta$ and $\theta'$ must lie in the same $H_n$-orbit.
\end{proof}
\begin{defin}
    For $\beta\in \textup{Hom}(\{1, \ldots, n\}, \{0, \pm 1\})$, define
    $$
    V^a(\beta) = V\left(\sum_{i=1}^n \beta(i)a_i\right)
    $$
\end{defin}
\begin{proof}[Proof of Theorem \ref{thm:ZariskiOpenSet}]
    Let $\mathcal{H}=\Hom(\{1, \ldots, n\}, \{0,\pm 1\})-\{0\}$ Define the following Zariski open set:
    $$
    U = U_4 \cap U_\sim
    $$
    where 
    $$
    U_4 = \left(\bigcup_{\beta\in\mathcal{H}}V(\beta)\right)^c.
    $$
    We see that all parameters in $U$ satisfy the conditions of Proposition \ref{prop:mn1genericity} and so for all $\theta\in U, \rho^{-1}(\theta)$ admits the group structure of $H_n$. Thus $U\subset \Omega_{(m,n,1)}^{(1)}$, as desired.
\end{proof}
\begin{proof}[Proof of Theorem \ref{thm:maintheorem}]
    By Theorem \ref{thm:ZariskiOpenSet}, it just remains to show that $H_n$ is diffeomorphic to $\rho^{-1}_{(m,n,k)}(\theta)$ if the action of $H_n$ is simply transitive. However, supposing $H_n$ acts simply transitively on $\rho^{-1}_{(m,n,k)}(\theta)$, we see that $\rho^{-1}_{(m,n,k)}(\theta)$ is a homogeneous $H_n$-space, and that $\Stab_{H_n}\theta\simeq 1$. So by Lemma \ref{lemma:LeeLemmaGSpace}, $\rho^{-1}_{(m,n,k)}(\theta)\simeq H_n$, as desired.
\end{proof}

\printbibliography
\section{Code to Compute Minimal Form}\label{section:minimalformcode}
\begin{lstlisting}[language=Python]
    from math import sqrt
    #Notably this sign function returns 0 for 0 instead of 1 as might be more standard. 
    def sign(x):
        if x> 0: return 1
        elif x < 0: return -1
        else: return 0

    def isLinDep(x, y):
        PseudoOutput = 0
        l = len(x)
        for i in range(l):
            for j in range(i,l):
                PseudoOutput += abs(x[i]*y[j]-x[j]*y[i])
        if PseudoOutput == 0: return True
        else: return False
        
    def isSemiLinDep(x, y):
        xsign = [sign(a) for a in x ]
        ysign = [sign(a) for a in y]
        if isLinDep(x,y) and xsign == ysign: return True
        else: return False
        
    def inner(x):
        y = 0
        for a in x:
            y += a*a
        return y

    #This takes 4 different inputs but returns one input as a list in the form [v_min, A_min, b_min, c_min]. Notably there is no test for whether the parameter is actually a valid parameter which bears caution.
    def minimalForm(v, A, b, c):
        c_final = c
        n = len(v)
        Az = [0]*len(A[0])
        partitionRange = set(range(n))
        #This loop takes care of equivalences (1) and (3) in in the definition of the minimal form
        for i in range(len(b)):
            if A[i] == Az:
                c_final += v[i]*max(0, b[i])
                partitionRange.remove(i)
            if v[i] == 0:
                partitionRange.remove(i)
        combinedAb = [A[i]+ [b[i]] for i in range(n)]
        nonZeroSemiDepIndices = []

        #This loop determines linearly semi-dependent rows in A
        for i in partitionRange:
            if partitionRange == {}: break
            if i not in partitionRange: continue
            l = [j for j in partitionRange if isSemiLinDep(combinedAb[i],combinedAb[j])]   
            nonZeroSemiDepIndices.append(l)
            partitionRange = partitionRange - set(l)
        r_pos, r_neg, r_zero = [], [], n - len(nonZeroSemiDepIndices)

        #This loop takes care of equivalence (2) in the definition of the minimal form
        for I in nonZeroSemiDepIndices:
            if len(I) > 1:
                w = inner(combinedAb[I[0]])
                s = 1+ sum([(v[I[i]]/v[I[0]])*sqrt(inner(combinedAb[I[i]])/w) for i in range(1, len(I))])
                if v[I[0]]*s> 0:
                    r_pos.append([a*abs(s) for a in combinedAb[I[0]]])
                if v[I[0]]*s< 0:
                    r_neg.append([a*abs(s) for a in combinedAb[I[0]]])
                if s == 0:
                    r_zero += 1
            else:
                if v[I[0]]> 0:
                    r_pos.append([a*abs(v[I[0]]) for a in combinedAb[I[0]]])
                if v[I[0]] < 0:
                    r_neg.append([a*abs(v[I[0]]) for a in combinedAb[I[0]]])
        A_pos, A_neg, b_pos, b_neg = [], [],[],[]
        
        #These loops piece together everything to build the minimal form of the parameter
        for r in sorted(r_pos, reverse=True):
            b_pos.append(r.pop())
            A_pos.append(r)
        for r in sorted(r_neg, reverse=True):
            b_neg.append(r.pop())
            A_neg.append(r)
        v_final = [1]*len(r_pos) + [-1]*len(r_neg) + [0]*r_zero
        A_final = A_pos + A_neg + [Az]*r_zero
        b_final = b_pos + b_neg + [0]*r_zero
        
        return [v_final, A_final, b_final, c_final]
\end{lstlisting}

\end{document}